\documentclass[letterpaper]{article} 

\usepackage{aaai23}  
\usepackage{times}  
\usepackage{helvet}  
\usepackage{courier}  
\usepackage[hyphens]{url}  
\usepackage{graphicx} 
\usepackage{natbib}  
\usepackage{caption} 
\usepackage{booktabs}
\usepackage{algorithm}
\usepackage{algorithmic}
\usepackage{amsmath}
\usepackage{newfloat}
\usepackage{listings}
\DeclareCaptionStyle{ruled}{labelfont=normalfont,labelsep=colon,strut=off} 
\floatstyle{ruled}
\newfloat{listing}{tb}{lst}{}
\floatname{listing}{Listing}

\usepackage{amsthm,amsfonts,subfigure,xcolor,colortbl}
\usepackage{hyperref}
\newtheorem{theorem}{Theorem}
\newtheorem{criteria}[theorem]{Criteria}
\newtheorem{definition}[theorem]{Definition}

\newtheorem{proposition}[theorem]{Proposition}

\newenvironment{tightcenter}{%
  \setlength\topsep{1pt}
  \setlength\parskip{1pt}
  \begin{center}
}{%
  \end{center}
}

\title{Counterfactual Fairness Is Basically Demographic Parity}
\author{
    \textbf{Lucas Rosenblatt,$^\dagger$\textsuperscript{\rm 1}
    R. Teal Witter}$^\dagger$\textsuperscript{\rm 1} 
}

\affiliations{
    \textsuperscript{\rm 1}New York University\thanks{Research supported in part by NSF No. 1922658 and 1916505.}\\
    \{lucas.rosenblatt, rtealwitter\}@nyu.edu
}

\begin{document}

\maketitle

\def\thefootnote{$\dagger$}\footnotetext{Authors contributed equally.}\def\thefootnote{\arabic{footnote}}

\begin{abstract}

Making fair decisions is crucial to ethically implementing machine learning algorithms in social settings. In this work, we consider the celebrated definition of counterfactual fairness. We begin by showing that an algorithm which satisfies counterfactual fairness also satisfies demographic parity, a far simpler fairness constraint. Similarly, we show that all algorithms satisfying demographic parity can be trivially modified to satisfy counterfactual fairness. Together, our results indicate that counterfactual fairness is basically equivalent to demographic parity, which has important implications for the growing body of work on counterfactual fairness. We then validate our theoretical findings empirically, analyzing three existing algorithms for counterfactual fairness against three simple benchmarks. We find that two simple benchmark algorithms outperform all three existing algorithms---in terms of fairness, accuracy, and efficiency---on several data sets. Our analysis leads us to formalize a concrete fairness goal: to preserve the order of individuals within protected groups. We believe transparency around the ordering of individuals within protected groups makes fair algorithms more trustworthy. By design, the two simple benchmark algorithms satisfy this goal while the existing algorithms for counterfactual fairness do not.
\end{abstract}

\section{Introduction}
A pressing challenge in the deployment of algorithms in social
settings is the risk of ``unfair'' decisions with respect to 
protected attributes like race, gender, 
religion, sexual-orientation, or age.
In this work, we consider the supervised learning problem
of predicting outcomes from labelled observations
where the existing outcomes are already biased.
The observations consist of protected attributes and
other remaining attributes like test scores, grade point
averages, credit scores, health risks, and more context
dependent variables.
The goal is to ``fairly'' predict outcomes like
student success, credit worthiness, and health care needs.

One approach is to ignore the protected attributes
and train algorithms only on the remaining attributes i.e. ``fairness through blindness'' \cite{dwork2012fairness}.
However, there could be variables like height or zip code
which are proxies for protected attributes \cite{pedreshi2008discrimination,corbett2018measure}.
Even in the absence of these proxy variables, seemingly reasonable remaining predictive
attributes like study time or prior health care can still depend
on protected attributes.
A natural consequentialist approach in the face of concerns over group equity is demographic parity, which ensures an identical distribution over the outcomes for each protected group.
However, demographic parity has been critiqued because
it can allow for blatantly unfair choices to individuals \cite{dwork2012fairness, corbett2018measure}.

\begin{figure}[tb!]
    \centering
    \includegraphics[scale=.13,page=1]{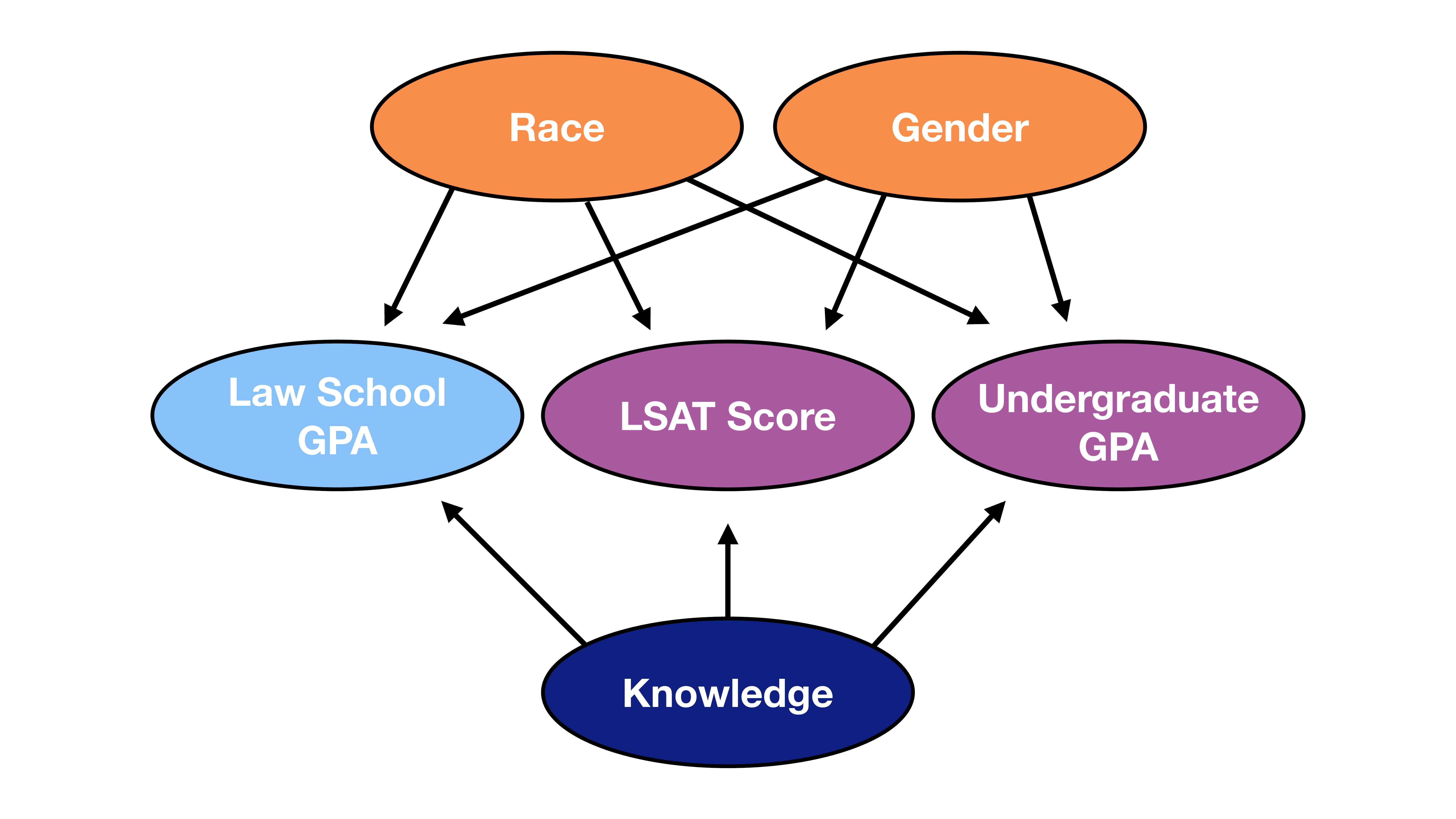}
    \caption{A causal model in a law school context.
    Directed edges between a source and a target
     denote a causal relationship in
     the sense that
     changing the value of the source
     (probabilistically)
     changes the value of the target.
     In this context, counterfactual fairness tries to make
     decisions based on knowledge (which cannot directly
     be measured) without biases from race or gender.}
    \label{fig:lawmodel}
    \vspace{-0.3cm}
\end{figure}

Counterfactual fairness offers a solution by advocating for a causal perspective, which provides statistical tools to help us minimize the direct effects of protected attributes on decisions
\cite{kusner2017counterfactual, Nilforoshan2022CausalCO}, and comports with a Rawlsian perspective of moral philosophy \cite{rawls2004theory}. It follows a surge of papers over the past decade
that borrow intuition from moral philosophy
to formally define what makes an algorithm fair
\cite{dwork2012fairness, feldman2015certifying, hardt2016equality, mitchell2021algorithmic}.
Drawing from \cite{pearl2000models}'s work on causal models,
\cite{kusner2017counterfactual} imagine a counterfactual world where only
the protected attributes of observations are changed.
Then, an algorithm is counterfactually fair if
it makes the same predictions for a hypothetical 
counterfactual version of
an observation as it would for the real version of the observation.
This individual notion of fairness promises to avoid the pitfalls
of the group based fairness measures we described earlier, though the relationship between causality, counterfactuals and moral philosophy is complicated \cite{kasirzadeh2021use}.

\vspace{-0.2cm}
\subsection{Preliminaries and Notation}

We follow notation laid out by
\cite{kusner2017counterfactual} for our formal analysis.
The variables are partitioned into $A$ (the set of
protected attributes), $U$ (the set of latent
attributes), $X$ (the set of remaining
variables), and $Y$ (the outcomes).

Figure \ref{fig:generalcausalmodel} shows a general causal
model.
Each variable is generated (probabilistically) 
from a governing distribution.
The parameters of a variable's distribution are fixed or derived
from the incoming edges in the causal model (or both).
For example, 
we may have $a \sim \textrm{Binom}(.3, 1)$, $u \sim \mathcal{N}(0,1)$,
and $x \sim \textrm{Pois}(au)$.

 \begin{figure}[ht]
     \centering
     \includegraphics[scale=.23,page=1]{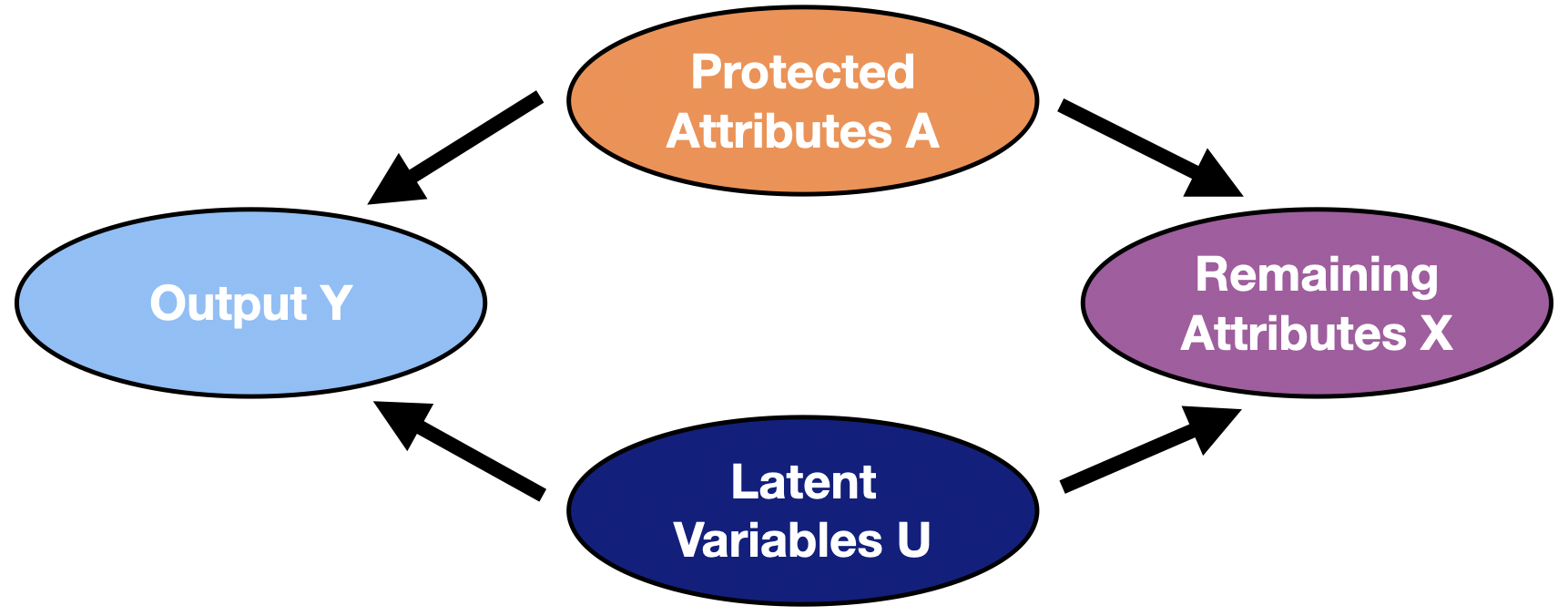}
     \caption{A general causal model consists of a directed acyclic graph
     and governing equations for the (probabilistic) generation of each
     variable.}
     \label{fig:generalcausalmodel}
 \end{figure}

A counterfactually fair algorithm uses $U$
to predict $Y$ since every 
other attribute is influenced
by the protected attributes $A$.
However, $U$ cannot be measured directly.
The solution of counterfactual fairness
is to model the 
joint distribution of each attribute 
in $X$ induced by the values of $A$ and $U$.
Then $U$ is estimated by the posterior distribution
on instances of $A$ and $X$.

Building a causal model requires domain expert knowledge
to establish both the direction of relationships between
attributes and the family of distributions (e.g., normal,
Poisson, binomial, etc.) that generates the 
distribution of each variable.
Once a causal model is determined, it is computationally
intensive to learn the parameters of each distribution
and then to estimate the posterior distribution of each attribute in $U$.



\vspace{-0.2cm}
\subsection{Fairness Definitions}

In this section, we formalize the definitions
of demographic parity and counterfactual fairness.
We use capital letters $A$, $X$, $U$, and $Y$
to refer to sets of random variables,
and we use lowercase letters
$a$, $x$, $u$, and $y$ to refer to realizations 
of these random variables. 
Consider a predictor
$\hat{Y}:A \times U \times X \rightarrow Y$.
If a predictor $\hat{Y}$ does not use a particular
set of random variables then we omit that set.
We use $\hat{Y}_{A \gets a}(u)$ for some protected
group $a \in A$ to denote the prediction on 
$u$ where attributes in $A$ are set to $a$.

\begin{definition}[Demographic Parity \cite{calders2009building,zliobaite2015relation,zafar2015learning}]\label{def:demographicparity}
Given a predictor $\hat{Y}: X \times A \rightarrow Y$,
we say $\hat{Y}$ satisfies demographic parity if,
for all instances of protected attributes
$a$ and $a'$,
\begin{align*}
    \Pr(\hat{Y}(x, a) = y|A = a)=
    \Pr(\hat{Y}(x, a') = y| A = a')  
\end{align*}
where the probability is taken over
the conditional distribution of $X$
and the possible randomness of $\hat{Y}$.
\end{definition}

\begin{definition} [Counterfactual Fairness 
\cite{kusner2017counterfactual}]
\label{def:counterfactualfairness}
Let $A$, $U$, $X$, and $Y$ be sets of random
variables in a given causal model
with specified distributions for each attribute.
Consider any protected group $a' \in A$
and any outcome $y \in Y$.
Then a predictor
$\hat{Y}:U \times A \rightarrow Y$ is counterfactually fair 
if for all observations $x \in X$ and $a \in A$,
\begin{align*}
 \Pr(\hat{Y}_{A\leftarrow a}(u) = y| X=x, A=a) \\
 = \Pr(\hat{Y}_{A\leftarrow a'}(u) = y| X=x, A=a)
\end{align*}
where the probability is with respect to the posterior
distribution of $u$ induced by fixed observations $x$ and $a$.
\end{definition}

Definition~\ref{def:counterfactualfairness} states that it must be the case that the final prediction
of a counterfactually fair predictor
$\hat{Y}_{A \gets a}(u)$ is distributed like
$\hat{Y}_{A \gets a'}(u)$ when $a \neq a'$ and where the distribution of $u$ is
induced by observations of $X$ and $A$.

Now we have established the two primary definitions of fairness that we consider in our work. However, it is worth noting that difficulties in defining algorithmic fairness have led to many competing approaches. Furthermore, formal results have demonstrated the incompatibility of applying some fairness definitions in conjunction with others, which hints at irreconcilable philosophical groundings and a need for formalizations for different policy settings \cite{kleinberg2016inherent, chouldechova2016fair, corbett2018measure}.

\vspace{-0.2cm}
\subsection*{Our Contributions}
Our work starts with a simple observation
on the distribution of latent variables:
by their definition, latent variables must be
independent of protected attributes for the model to be counterfactually fair.
A careful look at this statement suggests
a connection to demographic parity, which
requires that predictions must be
independent of protected attributes.
A natural question is then:
\begin{tightcenter}
Do algorithms that satisfy demographic parity 
also satisfy counterfactual fairness?
\end{tightcenter}
But for an algorithm to even satisfy the definition
of counterfactual fairness, we need a causal model,
latent variables, and a way of estimating latent
variables.
Our first result answers the question in the affirmative
if we trivially construct the necessary counterfactual
fairness machinery.

The next natural question is whether the converse
holds:
\begin{tightcenter}
Do algorithms that satisfy counterfactual fairness 
also satisfy demographic parity?
\end{tightcenter}
We answer this question in the affirmative, as well.\footnote{
After posting an early draft of this work, we were made aware of a paper released several months before which independently shows counterfactual fairness implies demographic parity (see section 5.3 of \cite{fawkes2021selection}).
We believe this contemporary interest points
to the importance of our work.}
Together, the two results suggest that counterfactual
fairness is basically demographic parity.
This relationship is important for the fairness community
since counterfactual fairness is
often celebrated as a novel tool for fair decisions
while demographic parity is considered simple and flawed. For example, in \cite{caton2020fairness}, they note that demographic parity is flawed because it does not account for potential differences between groups, but do not mention limitations of counterfactual fairness. Other influential work takes a similar position \cite{gajane2017formalizing,corbett2018measure,coston2020counterfactual}.

We also analyze three existing algorithms for
counterfactual fairness in comparison to three
benchmark algorithms.
The data set is based on the running law school example
and we corroborate our findings on two additional
data sets and causal models.
Our analysis uses a relaxation of counterfactual fairness
and Kruskal-Wallis H tests (as one-sided tests for independence).
We find that two simple benchmark algorithms outperform
the existing algorithms in terms of
satisfying the requirements of counterfactual fairness,
computational efficiency, and accuracy.

We also introduce a notion of preserving group ordering.
Put simply, it requires that an individual who performs
better than another individual in the \textit{same}
protected group
under unfair labels must also perform better than
that individual under fair labels.
We believe that this definition is important 
for transparency and consistency, laying the foundations
for trust in fair algorithms. Why? Because inexplicable decisions in purportedly ``fair'' decision making processes can mask unforeseen technical bias \cite{abdollahi2018transparency, bhatt2020explainable}.
Interestingly, we show that preserving group ordering
can be mutually exclusive with counterfactual fairness
for certain causal models.
Empirically, we find the existing algorithms
for counterfactual
fairness have remarkably unstable orderings while
the simple benchmark algorithms are
consistent by design.

\paragraph*{Outline}
\begin{itemize}
    \item We first show that all counterfactually fair predictors
    satisfy demographic parity and all predictors satisfying demographic parity
    can be trivially modified to satisfy 
    counterfactual fairness.
    \item We subsequently analyze six algorithms (including
    those presented in \cite{kusner2017counterfactual})
    in the context of the following two questions:
    \begin{itemize}
        \item How do these algorithms satisfy three special
        fairness constraints (counterfactual fairness, demographic parity, and the independence required for latent variables) while still maintaining reasonably accurate predictions?
        \item How do these algorithms compare 
        on the predictions they make at the individual level?
    \end{itemize}
\end{itemize}
We corroborate our empirical results on additional data sets
in the context of healthcare and loans.
All our code and results are available online\footnote{
github.com/lurosenb/simplifying\_counterfactual\_fairness}.


\vspace{-0.1cm}
\subsection*{Related Work}
Since its introduction, counterfactual fairness has been highly influential. Many recent works on equity in statistics and machine learning directly use or modify only slightly Definition~\ref{def:counterfactualfairness} to enable training classifiers and other decision making algorithms  \cite{Zhang2018FairnessID,Chiappa2019PathSpecificCF,Wu2019CounterfactualFU, coston2020counterfactual,Black2020FlipTestFT,mhasawade2021causal,Chikahara2021LearningIF,vonKgelgen2022OnTF}. \cite{Zhang2018FairnessID} create a procedure for identifying discrimination and applying causal explanations, and then use counterfactuals to design repairs while offering evaluations of this system on a host of examples. \cite{Wu2019CounterfactualFU} and \cite{Chiappa2019PathSpecificCF} expand theoretical groundings for counterfactual fairness, with the former improving on the latent attribute identifiability by offering bounds and techniques. Prior work also expands counterfactual fairness to deal with 
between-group rankings \cite{yang2020causal}. Furthermore, task specific variations and implementations of Definition~\ref{def:counterfactualfairness} also exist for the medical domain \cite{Pfohl2019CounterfactualRF}, variational autoencoders in computer vision \cite{Kim2021CounterfactualFW}, generative-adversarial networks \cite{Xu2019AchievingCF}, and even natural language processing \cite{sari2021counterfactually}.

There are several existing concerns about
counterfactual fairness. Recent work by \cite{Nilforoshan2022CausalCO} present a result that unites many causal notions of fairness, and further cautions that a gap exists between the effects of these popular approaches on fairness and their consequences, highlighting both practical and mathematical limitations.
With particular relevance to our work, they show that when
a predictor satisfying demographic parity \textit{and} a causal
model are specified, the predictor does not necessarily satisfy
counterfactual fairness (see Theorem 2).
In contrast, our work considers the setting where we are only given
a predictor satisfying demographic parity (not a causal model) and
can build a causal model of our choice.
In \cite{kasirzadeh2021use}, the authors conclude that counterfactual fairness may require what they deem to be an ``incoherent theory'' of social categories, i.e., some social categories may not admit counterfactual manipulation.
In \cite{kilbertus2020sensitivity}, they examine the sensitivity of counterfactual fairness to ``unmeasured confounding.'' This is when a true causal effect between variables can be at least partially described by a non-zero correlation between two of the $\epsilon$-``error'' variables calculated during one possible counterfactual fairness process, which can affect the independence guarantee from $A$. We position our work as an important addition to this existing set of papers expressing concerns with and limitations of counterfactual fairness.

Better understanding the theoretical and practical consistency of algorithmic fairness has real social impact. As fairness is operationalized through open source libraries like Fairlearn \cite{bird2020fairlearn} and AIF360 \cite{bellamy2019ai}, work that furthers our understanding of fair model guarantees becomes increasingly important.

\section{Demographic Parity Is Counterfactual Fairness \label{sec:dempariscf}}

Our results start with a simple observation
about the definition of latent variables.
In the original counterfactual fairness paper,
the authors define latent variables $U$ so that 
``$U$ is a set of latent background variables,
which are factors not caused by any variable in
the set $V[=A \cup X]$ of observable variables''
where $A$ is the set of protected attributes
and $X$ is the set of remaining variables
\cite{kusner2017counterfactual}.
In other words, there are no incoming edges to $U$
in the causal model.
This means that instances of latent variables
are generated from a fixed set of parameters
(that do not depend on protected attributes).
In addition, while not explicitly stated in the definition, all causal models we
are aware of also do not have any incoming edges to $A$.
This means that instances of protected attributes
are generated from a fixed set of parameters
(that do not depend on latent variables).
Together, these observations imply the following:

\begin{criteria}\label{crit:independence}
Latent variables are independent
of protected attributes.
\end{criteria}

This criteria is key to understanding practical applications of
counterfactual fairness.
Counterfactual fairness attempts
to explain behavior from protected attributes
and latent variables.
The protected attributes hold information on
features \textit{outside} of individuals' control
which would very often be unfair to base decisions on
while the latent variables contain information on
features \textit{within} an individual's control which
are fair to base decisions on.
From this perspective, it is clear that latent
variables must be independent of protected
attributes. 
Otherwise, this would imply that the inherent `worthiness'
of latent variables would problematically depend
on protected attributes like race, gender, and age.

With Criteria \ref{crit:independence} in hand,
we show that any predictor which satisfies demographic
parity also satisfies counterfactual fairness
after a trivial modification.

Note that a predictor which satisfies
counterfactual fairness must have an internal method
of estimating latent variables and a causal model.
So to say that a predictor which satisfies demographic
parity also satisfies counterfactual fairness,
we must introduce such a method and causal model.
We do so in the proof of the first theorem in a trivial way.

\begin{theorem}\label{thm:dpimplescf}
Any predictor $\hat{Y}': X \times A \rightarrow Y$
that satisfies demographic parity can
be modified into a method for estimating
latent variables and a predictor 
$\hat{Y}: U \times A \rightarrow Y$
that is counterfactually fair.
\end{theorem}

\begin{proof}
    Our estimate of the latent variables $u$, given
    protected attributes $a$ 
    and remaining attributes $x$,
    is $u = \hat{Y}'(x, a)$.
    Then the predictor $\hat{Y}: U \times A \rightarrow Y$
    is given by the identity function
    $\hat{Y}_{A \gets a}(u) = u$.
    Because $\hat{Y}'$ satisfies demographic parity,
    we know
    $\Pr(u |a) = \Pr(u | a')$ for $a, a' \in A$.
    That is, latent variables are independent
    of protected attributes which meets Criteria \ref{crit:independence}.
    Further, $\hat{Y}$ satisfies counterfactual
    fairness because $\hat{Y}_{A \gets a}(u) 
    = u = \hat{Y}_{A \gets a'}(u)$.
\end{proof}

The next theorem proves the converse
of Theorem \ref{thm:dpimplescf}; namely,
any counterfactually fair predictor also satisfies
demographic parity.

\begin{theorem}\label{thm:main}
    Consider a method for estimating latent
    variables and a counterfactually fair
    predictor $\hat{Y}:U \times X \rightarrow Y$.
    Then the resulting predictor 
    $\hat{Y'}: X \times A \rightarrow Y$
    satisfies demographic parity.
\end{theorem}

\begin{proof}
    Since the predictor is counterfactually fair,
    we know 
    \begin{align}\label{eq:cf}
       \Pr(\hat{Y}_{A \gets a}(u)=y|x, a)
        =\Pr(\hat{Y}_{A \gets a'}(u)=y|x, a) 
    \end{align}    
    for all $y \in Y$ where the realization
    of latent variables
    $u \in U$ is estimated from observations
    $a \in A$ and $x \in X$.
    Taking a weighted sum (if $X$ contains
    continuous random variables, an analogous statement
    holds via integration) of the right side of
    Equation (\ref{eq:cf}) yields
    \begin{align*}
        &\sum_{x \in X} \Pr(x | a)
        \Pr(\hat{Y}_{A \gets a}(u)=y|x, a) \\
        &= \sum_{x \in X} \Pr
        (x, \hat{Y}_{A \gets a}(u)=y|a) \\
        &= \Pr(\hat{Y}_{A \gets a}(u)=y|a)
        = \Pr(\hat{Y}_{A \gets a}(u)=y).
    \end{align*}
    The first equality holds because
    $\Pr(x|a) \Pr(y|x,a) = \Pr(x,y|a)$
    for random events $x,y,a$.
    The last equality holds by Criteria 
    \ref{crit:independence} and since the
    (possible) randomness of $\hat{Y}_{A \gets a}(u)$
    only comes from $u$ and the assignment $A \gets a$.
    The final equality gets us close to demographic parity
    but the predictions could still depend on the
    assignment.
    We repeat the same steps for the
    left side of Equation (\ref{eq:cf})
    and find that
    \begin{align*}
        \Pr(\hat{Y}_{A \gets a}(u) = y)
        = \Pr(\hat{Y}_{A \gets a'}(u) = y).
    \end{align*}
    This tells us that the distribution of
    predictions made by a counterfactually fair
    algorithm are independent of protected
    attributes. That is, demographic parity holds.
\end{proof}

Note that the contrapositive of Theorem \ref{thm:main}
gives us a  simple way to test whether 
a predictor is counterfactually fair.
\vspace{-0.4cm}
\section{Algorithms for Counterfactual Fairness}\label{sec:heuristics}

In this section, we empirically analyze six algorithms in
the context of
counterfactual fairness on the running law school
example.
We compare them in terms of
how well they achieve demographic parity (Definition
\ref{def:demographicparity}),
counterfactual fairness (Definition
\ref{def:counterfactualfairness}), and
independence between protected attributes
and latent variables (Criteria \ref{crit:independence}).
We also compare the algorithms' test accuracy and the resulting
trade-offs between fairness and performance.
Finally, we investigate how predictions on individuals
differ between the algorithms.
\vspace{-0.1cm}
\subsection{The Algorithms}
The first three algorithms---referred to as
``Levels 1, 2 and 3''---come from the original counterfactual
fairness paper \cite{kusner2017counterfactual}.
The next two algorithms are simple heuristics
for demographic parity while the final algorithm
is a straightforward learner without fairness constraints.

For all the algorithms we analyze, we use a linear
regression model, $\ell_2$-norm loss,
and Adam optimizer for learning.
We describe the algorithms next.

\paragraph*{Level 1} Since evaluating the relationships
between variables in a causal model is computationally
expensive, the first level only uses
the remaining variables that are independent of protected
attributes to learn the outcomes.
In practice, scenarios where remaining variables are not at least
partially conditioned on protected attributes are so rare as to be virtually non-existent.
Instead, we implement Level 1 by using all remaining variables
(without any protected attributes).
This approach has been called 
``fairness through unawareness'' \cite{grgic2016case}
and fails to make fair decisions because protected
attributes are often redundantly encoded
\cite{pedreshi2008discrimination}.
\paragraph*{Level 2} The second level uses the full power of causal models, but suffers from expensive and intensive computations; it requires domain expertise to specify the joint distributions over all variables for the causal model. Remaining variables are distributed according to subsets of latent variables and protected attributes. By using the known protected attributes, the remaining variables, and the causal model, Level 2 estimates likely values of latent variables. Then only the latent variables are used to learn a predictor over the set of possible outcomes.
\paragraph*{Level 3} The third level is a compromise between the simplicity of Level 1 and complexity of Level 2. Level 3 uses the relationships in the causal model to express the remaining variables  of each individual as a deterministic function of related protected attributes and a special explanation term. The deterministic function is learned from protected attributes to explain the remaining variable. Then the difference between the deterministic function and an individual's remaining variable is their explanation term. These explanation terms are used to learn the outcomes.

\begin{lstlisting}[caption={Converting unfair predictions to fair predictions when labels are normally distributed.},label=alg:normal,captionpos=t,float,abovecaptionskip=-\medskipamount, mathescape=true,escapeinside={*}{*}]
Input: $n$ group identities $a_i \in A$ and unfair labels $\bar{y_i}$ for samples $i \in [n]$
Output: predictions $\hat{y}_i$ that satisfy demographic parity
$\mu \leftarrow $ mean($\{ \bar{y}_i: i \in [n] \}$)
$\sigma \leftarrow $ std($\{ \bar{y}_i: i \in [n] \}$)
for $a$ in $A$ do
    $\mu_a \leftarrow $ mean($\{ \bar{y}_i: a_i=a \}$)
    $\sigma_a \leftarrow $ std($\{ \bar{y}_i: a_i=a \}$)
    for $i$ in $\{i: a_i = a\}$ do
        $\hat{y_i} = \mu + \sigma (\bar{y}_i - \mu_a)/\sigma_a$
\end{lstlisting}

\begin{lstlisting}[caption={Converting unfair predictions to fair predictions under any distribution of labels.},label=alg:quantile,captionpos=t,float,abovecaptionskip=-\medskipamount,mathescape=true,escapeinside={*}{*}]
Input: $n$ group identities $a_i \in A$ and unfair labels $\bar{y_i}$ for samples $i \in [n]$
Output: predictions $\hat{y}_i$ that satisfy demographic parity
$CDF \gets $ empirical CDF of $\{y_i: i \in [n]\}$
for $a$ in $A$ do
    $CDF_a \gets $ empirical CDF of $\{y_i: a_i = a \}$
    for $i$ in $\{i: a_i = a\}$
        $\hat{y}_i \gets CDF( CDF_a^{-1}(\bar{y}_i))$
\end{lstlisting}

The next two algorithms are simple heuristics
for demographic parity.
\paragraph*{Listing~\ref{alg:normal}} The first listing assumes the distribution of outcomes in every protected group is normally distributed. Then we can achieve demographic parity by simply calculating the normalized score of each individual within their protected group and converting it to the outcome  distribution of the full population.
\paragraph*{Listing~\ref{alg:quantile}} The second listing drops the assumption on the distribution of each protected group. Instead, Listing 2 estimates the cumulative density function (CDF) of each protected group and converts an individual's relative position within their group to the same position within the full population.

Listing 1 is appropriate when the outcomes are normally distributed
while Listing 2 is appropriate in general but cannot distinguish
noise from signal in distributions.
\paragraph*{Full Linear Model} The final algorithm uses a linear model on all protected attributes and remaining variables to learn outcomes. We expect the full linear model to mimic the unfair behavior of the underlying data but achieve the best test accuracy.

\begin{figure}[ht]
    \centering
    \begin{subfigure}
        \centering
        \includegraphics[scale=.27]{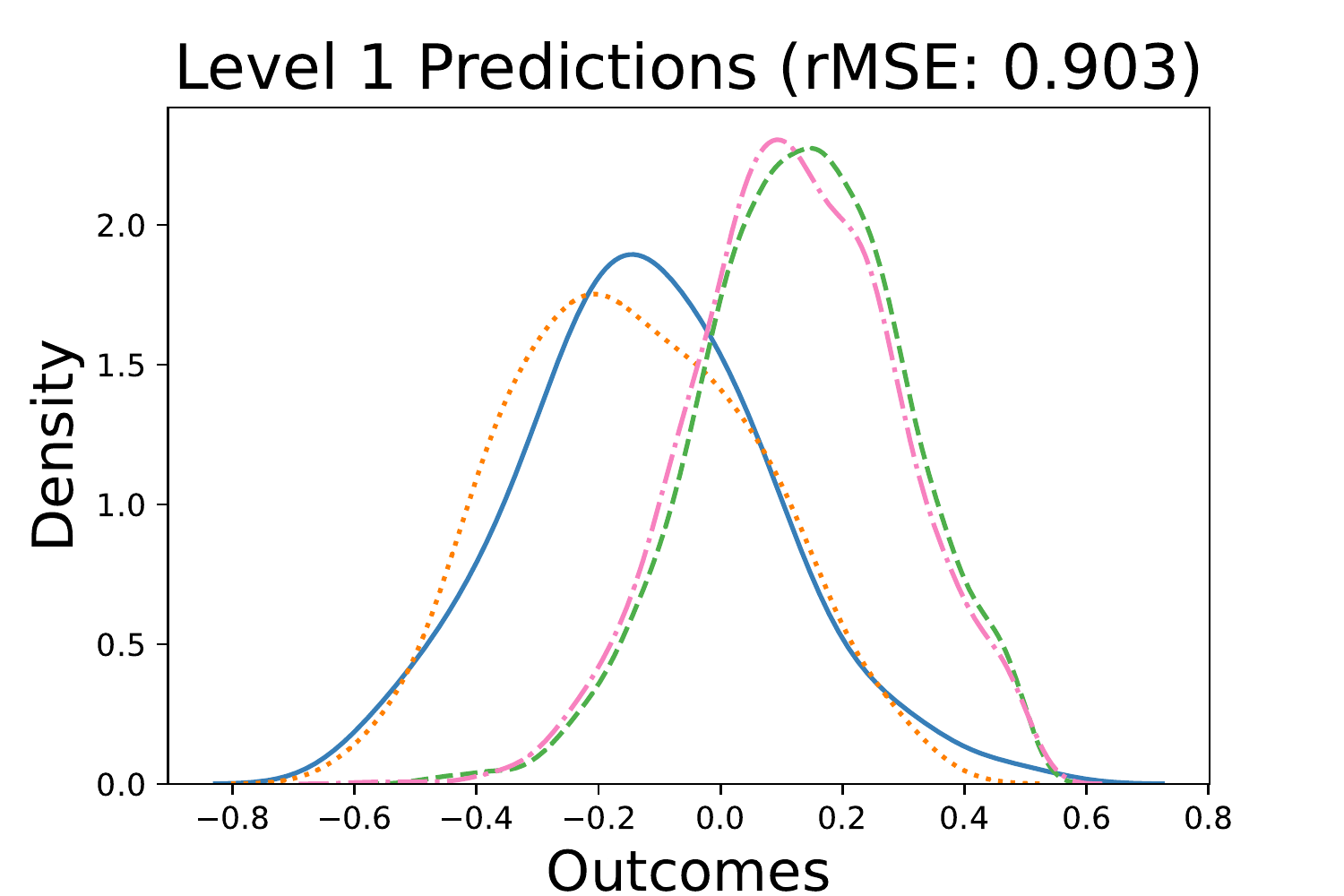}
    \end{subfigure}   
    \begin{subfigure}
        \centering
        \includegraphics[scale=.27]{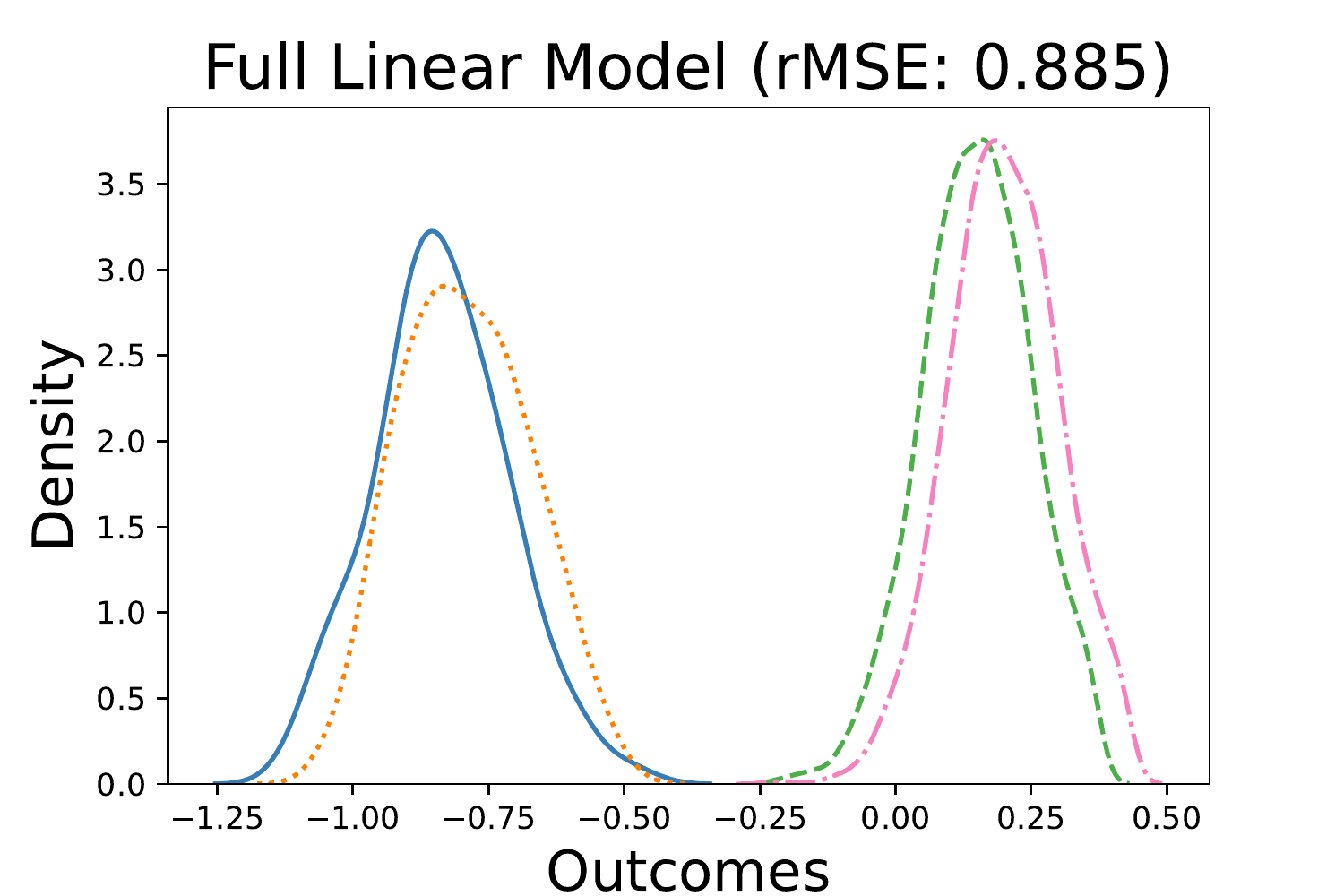}
    \end{subfigure}
     \begin{subfigure}
        \centering
        \includegraphics[scale=.27]{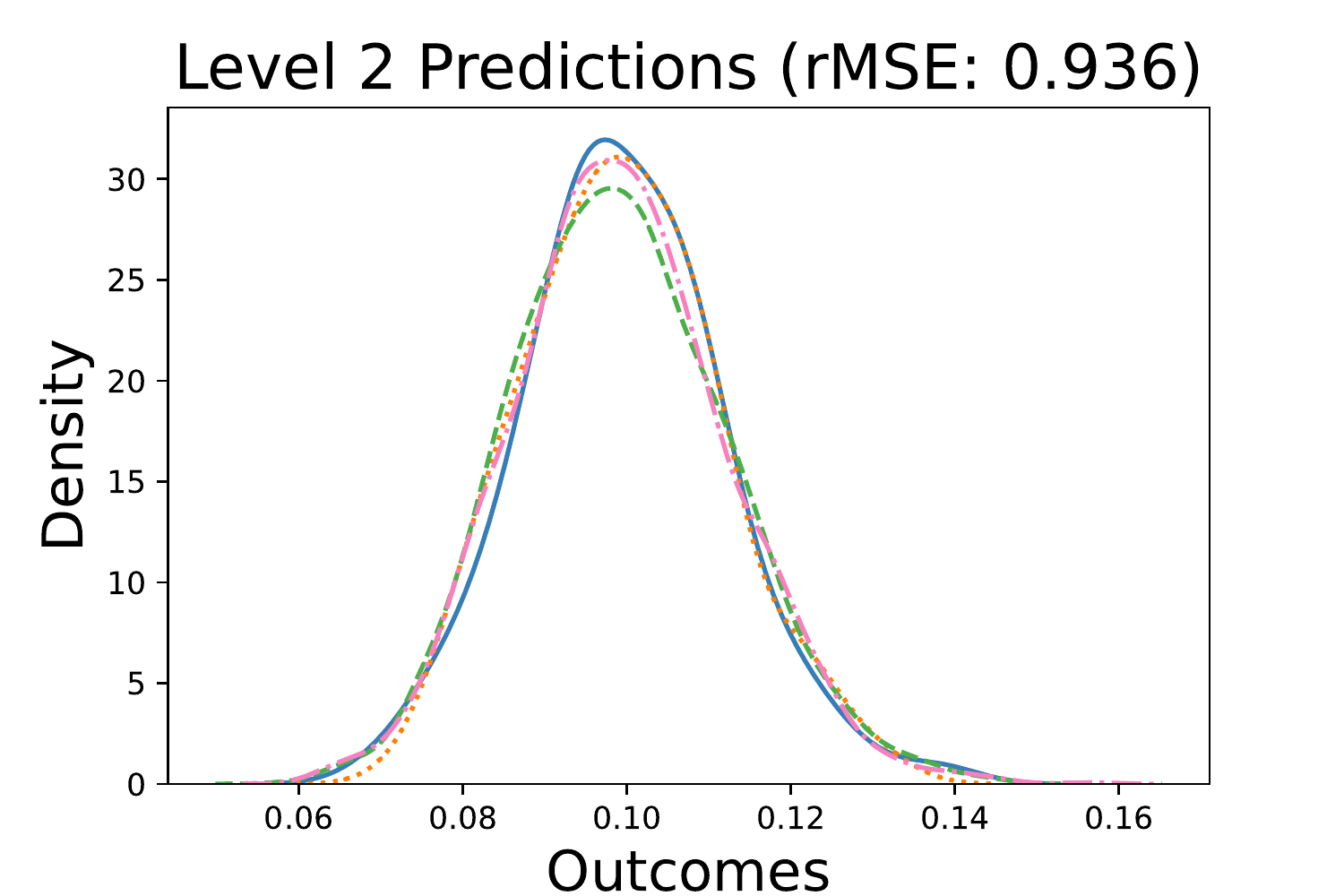}
    \end{subfigure}    
    \begin{subfigure}
        \centering
        \includegraphics[scale=.27]{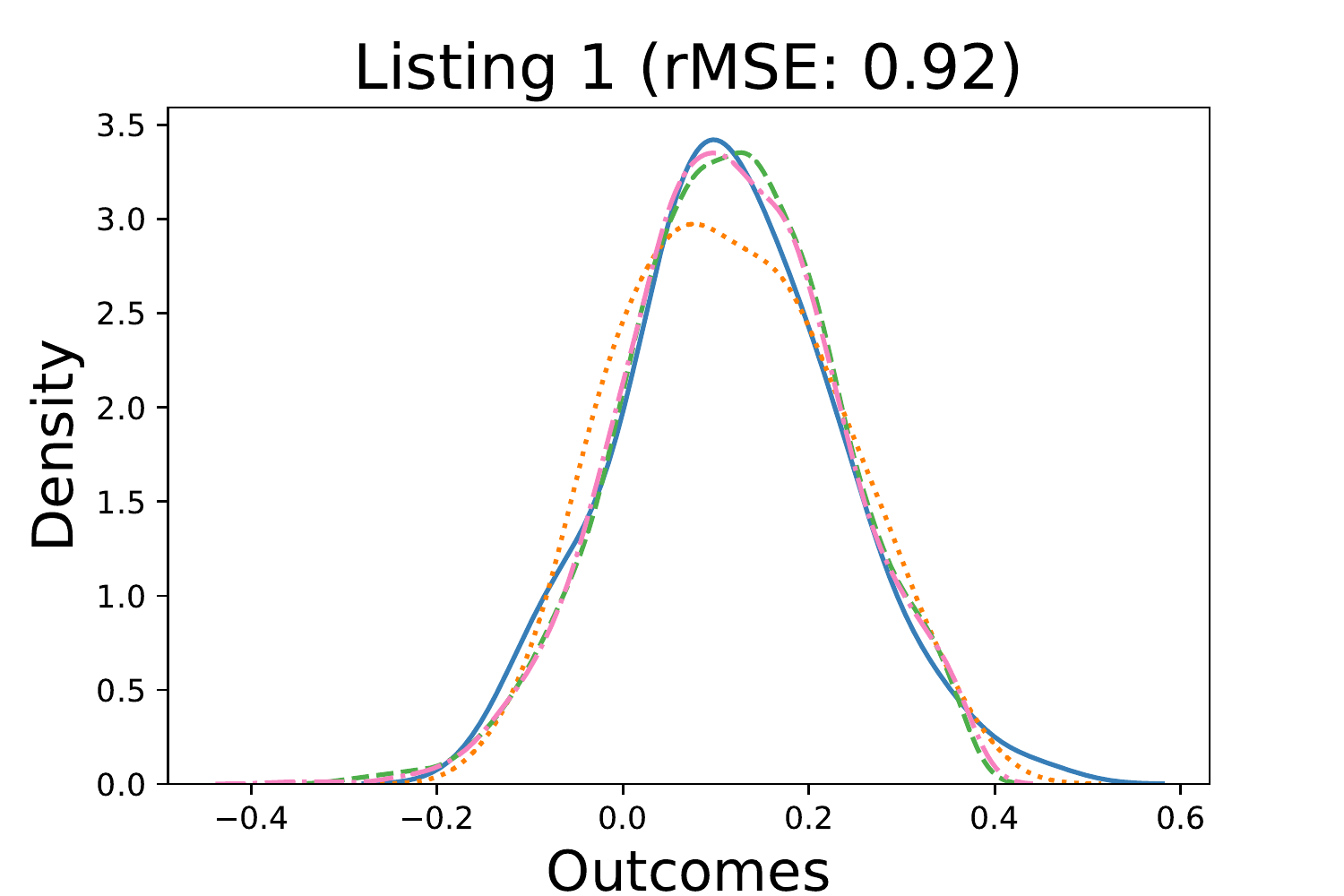}
    \end{subfigure} 
    \begin{subfigure}
        \centering
        \includegraphics[scale=.27]{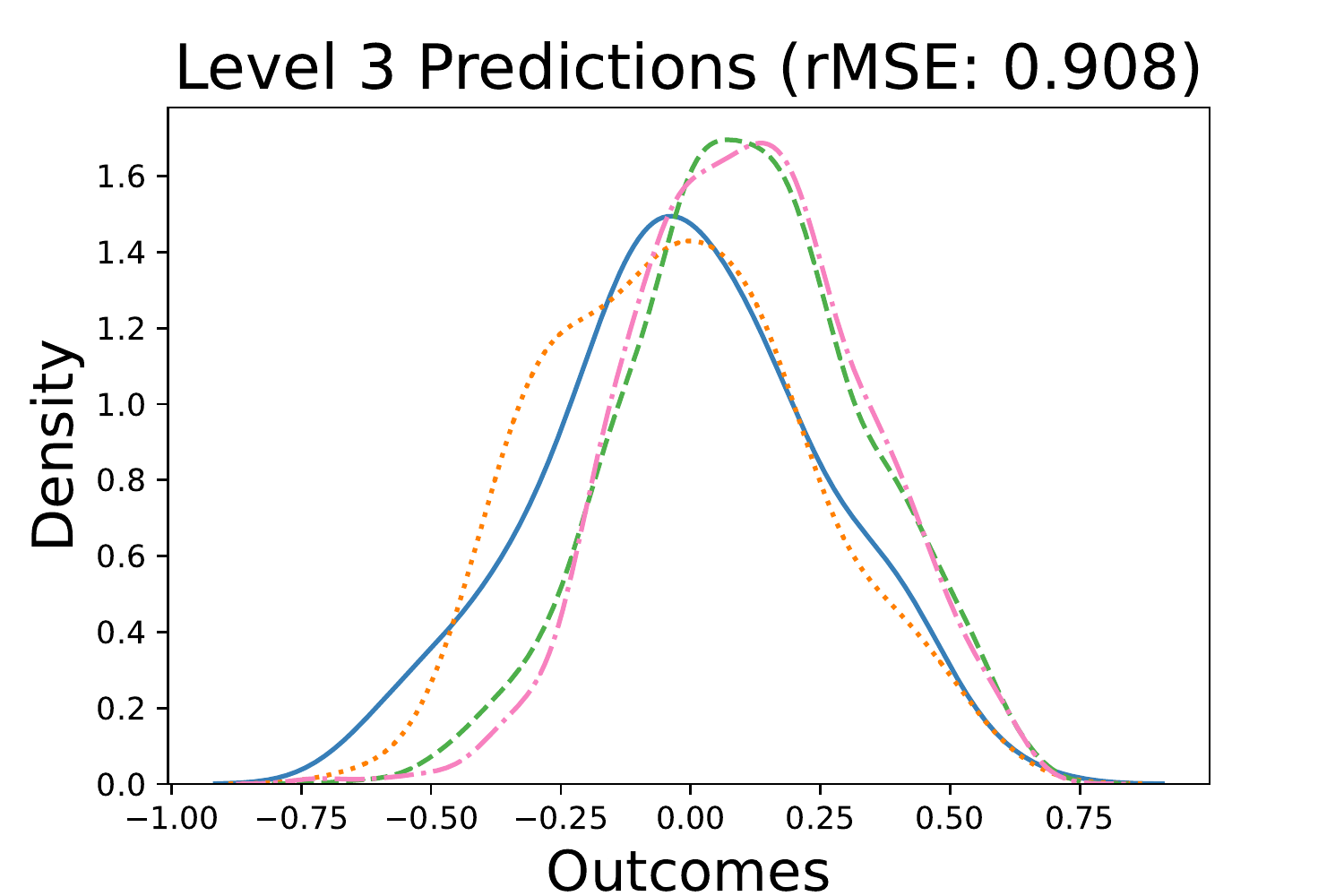}
    \end{subfigure} 
    \begin{subfigure}
        \centering
        \includegraphics[scale=.27]{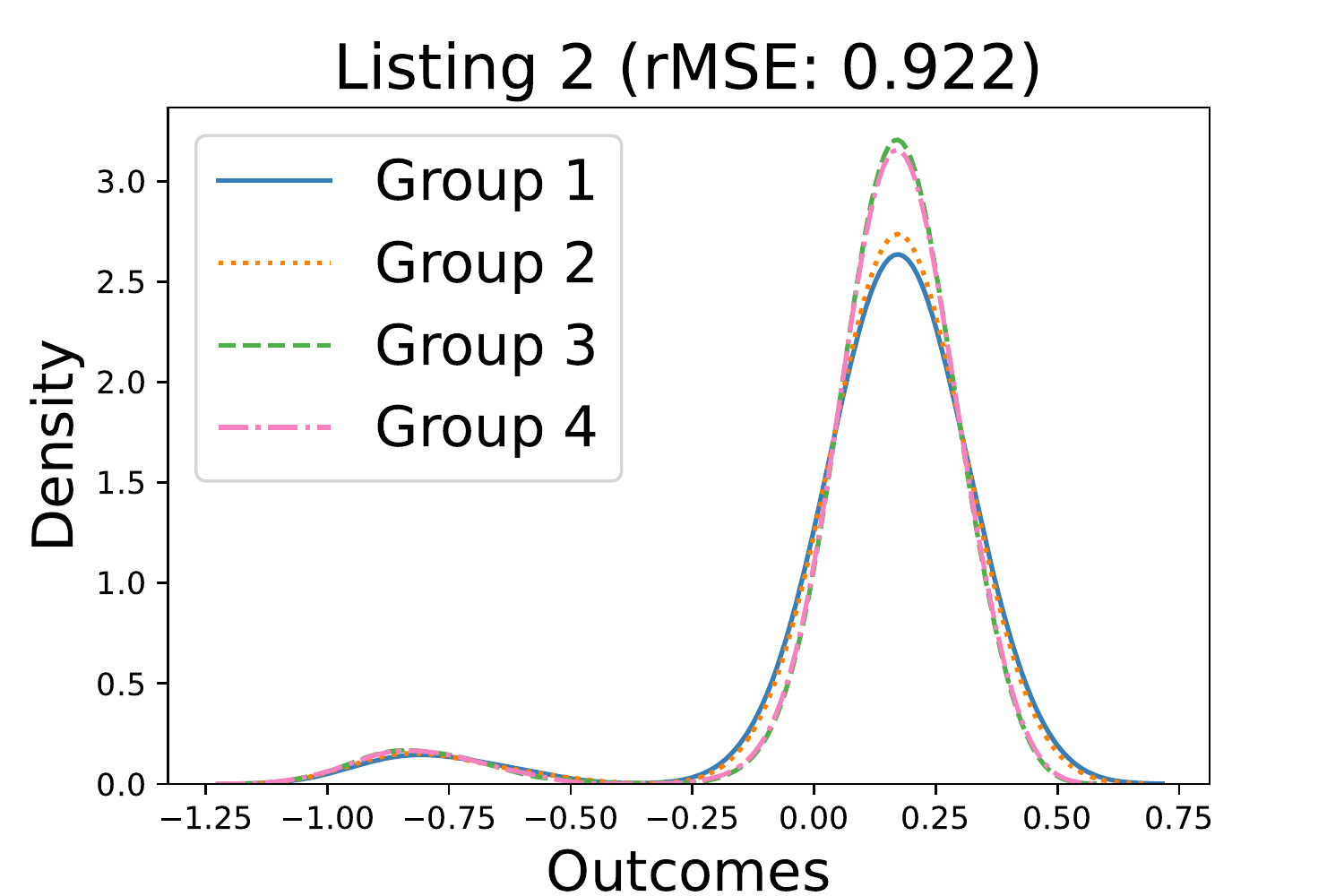}
    \end{subfigure}    
    \caption{Density plots for the distributions over predictions made by the six
    algorithms we consider. The root mean squared error
    (rMSE) is given next to each title and in Table \ref{tab:rmse}. Groups 1, 2, 3, and 4 correspond to protected groups over the Cartesian product of binary race and gender.}
    \label{fig:alg_predictions}
\end{figure}
\vspace{-0.2cm}
\subsection{Measures of Fairness}
We use several empirical tests to measure how well
each of the algorithms achieves different definitions
of fairness.

We first define a relaxation of counterfactual fairness
introduced in \cite{russell2017worlds}.
\begin{definition}[$(\epsilon, \delta)$-Approximate Counterfactual Fairness (ACF)]
    A predictor $\hat{Y}:X \times A \rightarrow Y$ 
    satisfies $(\epsilon, \delta)$-ACF if 
    \begin{align*}
        \Pr(\vert \hat{Y}_{A \gets a}(u) 
        -\hat{Y}_{A \gets a'}(u)  \vert \le \epsilon \vert
        x, a) \ge 1-\delta.
    \end{align*}
\end{definition}
Notice that calculating $(\epsilon, \delta)$-ACF
requires an internal estimate of latent variables.
The only algorithms that have a nontrivial estimate of latent
variables are Level 2 and Level 3.
Since both Level 2 and Level 3 use \textit{only} the
latent variables in learning the outcomes,
changing the protected attribute has no effect and
both algorithms are $(0,0)$-ACF.

For testing demographic parity and whether 
latent variables are independent
of protected attributes, we use the 
Kruskal-Wallis H test \cite{kruskal1952use}.
\begin{definition}[Kruskal-Wallis H Statistic]
    The H statistic is given by
    \begin{align*}
        (N-1)
        \frac{\sum_{i=1}^g n_i 
        (\bar{r}_i - \bar{r})^2}
        {\sum_{i=1}^g \sum_{j=1}^{n_i}
        (r_{ij}-\bar{r})^2}
    \end{align*}
    where $N$ is the total number of observations,
    $g$ is the number of groups,
    $n_i$ is the number of observations in group $i$,
    $r_{ij}$ is the rank (among all observations)
    of observation $j$ from group $i$,
    $\bar{r}_i$ is the average rank of observations
    in group $i$,
    and $\bar{r}$ is the average rank of
    all observations.
\end{definition}

Using the Kruskal-Wallis H statistic,
we can calculate the probability that every
group comes from a distribution with the same
median.
Since the latent variables are independent 
of the protected
attributes if and only if the 
distributions of the protected
groups are the same, a small probability
from the Kruskal-Wallis H test implies that the
latent variables and protected attributes
are not independent.
We also use the test to 
determine whether an algorithm satisfies
demographic parity.
Note it is possible that all groups
come from distributions with the same median
but that the distributions are in fact different
so the test has one-sided error.

We present our results of the Kruskal-Wallis
H tests in Table \ref{tab:fair}.
The first three data rows tell us the probability that
Criteria \ref{crit:independence} holds.
It is very likely that the latent variables in
each protected group come from distributions
with the same median in Level 2, suggesting that
Level 2 satisfies Criteria \ref{crit:independence}
in addition to counterfactual fairness.
In contrast, it is unlikely (probability $< .05$) that the 
two latent variables 
in Level 3 are independent of protected attributes,
suggesting that Level 3 does 
not satisfy Criteria~\ref{crit:independence}.
Note that Level 3 is an example of an algorithm
that ostensibly satisfies counterfactual 
fairness but
the latent variables do not meet the necessary
condition of Criteria \ref{crit:independence}.
We corroborate this finding in Figure 
\ref{fig:alg_predictions}:
Level 1 and Level 3
clearly do not satisfy demographic parity
because the latent variables are not 
independent of protected attributes.

The last six rows in Table \ref{tab:fair} tell us whether
each algorithm satisfies demographic parity.
As expected, Level 1 and the Full Linear Model are
\textit{remarkably} unlikely to satisfy demographic parity.
In line with the results for latent variables,
we see that Level 3 also is unlikely to satisfy demographic parity.
Finally, Level 2 likely satisfies demographic 
parity.
This is to be expected since the assumptions
of Theorem \ref{thm:main}---namely, 
counterfactual fairness and 
Criteria \ref{crit:independence}---hold.
Of course, Listings 1 and 2 are very likely to
satisfy demographic parity by design.

\begin{table}[h!]
\centering
\small 
\begin{tabular}{lll}
\toprule
Variable & H Statistic & $p$-value \\
\midrule
Level 2 Latent Variable & $.347$ & $.951$ \\
Level 3 Latent UGPA & $10.8$ & $.0126$ \\
Level 3 Latent LSAT & $84.6$ & $3.21\times 10 ^{-18}$ \\
\midrule
Level 1 Predictions & $221$ & $9.81 \times 10^{-48}$ \\
Level 2 Predictions & $.346$ & $.951$ \\
Level 3 Predictions & $69.3$ & $5.97 \times 10^{-15}$ \\
Listing 1 Predictions & $.165$ & $.983$ \\
Listing 2 Predictions & $.035$ & $.998$ \\
Full Predictions & $829$ & $2.53 \times 10^{-179}$ \\
\bottomrule
\end{tabular}
\caption{Kruskal-Wallis H test results for different variables.
The $p$-value indicates the likelihood of observing the H
statistic if the variable was distributed the same in every
group.
Every number is reported to three significant figures.}
\label{tab:fair}
\end{table}

Table \ref{tab:rmse} gives the root mean squared error (rMSE)
of every algorithm we consider. As expected, the Full Linear Model
without fairness constraints has the lowest error.
Listing 1 and Listing 2 also give relatively low error.
In contrast, Level 2 has the highest error.
So far, Level 2 simulatenously satisfies
Definition \ref{def:demographicparity},
Definition \ref{def:counterfactualfairness}, and Criteria
\ref{crit:independence}
but we see this is at the expense of accuracy.

\begin{table}[h!]
\centering
\small 
\begin{tabular}{rrrrrr}
\toprule
Lvl 1 & Lvl 2 & Lvl 3 &  Lst 1 &  Lst 2 & Full \\
\midrule
.933 & .936 & .908 & .919 & .921 & .881 \\
\bottomrule
\end{tabular}
\caption{The root mean squared error (rMSE) for each algorithm predicted
to three significant figures.}
\label{tab:rmse}
\end{table}
\vspace{-0.3cm}
\subsection{Individual Predictions}
We have seen how the algorithms we consider perform
differently in terms of three fairness measures.
But how do the algorithms differ on the actual outcomes
for individuals?

Figure \ref{fig:changingrank} visualizes the change in relative ordering
of predicted outcomes for 40 randomly chosen individuals in Group 1
(Black men) in the law school example.
Notably, the relative ordering is constant between Listing 1, Listing 2,
and the Full Linear Model, as expected.
However, the relative ordering is \textit{highly} unstable for Level 1,
Level 2, and Level 3.
This suggests that while Level 2 and Level 3 
achieve counterfactual fairness,
they do so in dramatically different ways.
One explanation is that Level 2 satisfies Criteria \ref{crit:independence}
while Level 3 does not.

\begin{figure}[ht]
    \centering
    \includegraphics[width=8.5cm]{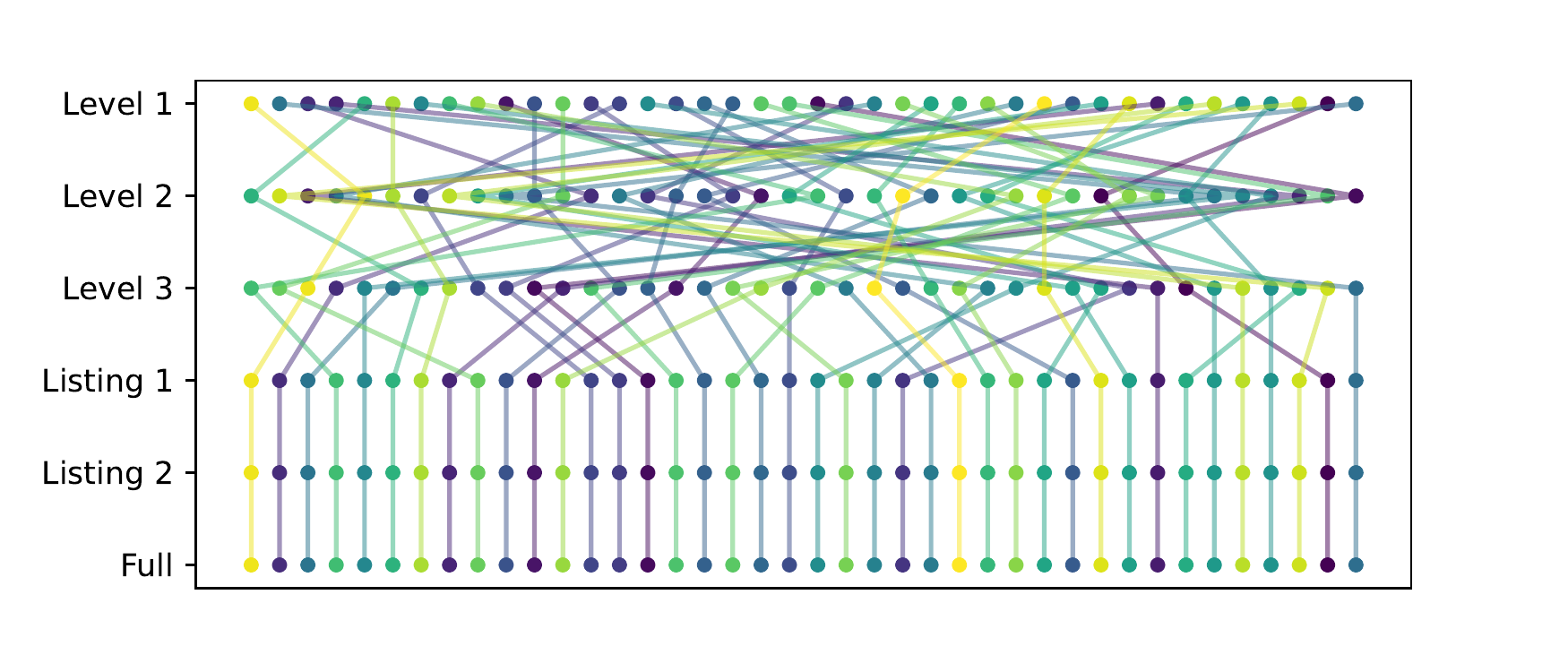}
    \caption{We randomly sample 40 individuals from Group 1
    (Black men) in the law school example and visualize how the relative 
    order of their predicted outcomes changes based on 
    the algorithm making predictions.}
    \label{fig:changingrank}
    \vspace{-0.2cm}
\end{figure}

We find it surprising that Level 2 does not agree with
Listing 1 and Listing 2 on the relative order of outcomes.
All three algorithms satisfy demographic parity, Criteria~\ref{crit:independence}, and, by Theorem \ref{thm:dpimplescf},
counterfactual fairness (when Listing 1 and Listing 2 are 
trivially modified).
How could it be that two algorithms are both fair under
the same definitions while making radically different
predictions on individuals?

This question suggests the following natural definition.

\begin{definition}[Preserving Group Ordering]\label{def:preserve}
    Given unfair labels, we say an algorithm preserves group ordering if the relative ordering  \textit{within} each protected group under the unfair labels is the the same as the relative ordering of individuals  induced by the algorithm's predictions.
\end{definition}

Put differently, if individuals are in the same protected class, then their relative ordering from a (possibly) unfair predictor should induce the same relative ordering as the fair predictor. That is, if an individual performs better under the same unfair conditions as another individual, they should perform better in a hypothetical world where those unfair conditions were removed.

We believe that preserved within-group orderings are preferable to the highly unstable orderings in the levels of  \cite{kusner2017counterfactual}. One reason to prefer preserved orderings is that even soft guarantees for relative orderings  of within group individuals leads to more transparent, trustworthy and ultimately \textit{useful} fair algorithms  \cite{abdollahi2018transparency, bhatt2020explainable}.

Observe that Listing 1 and Listing 2 satisfy 
Definition \ref{def:preserve} by design.
This contrasts with Levels 1, 2, and 3 
which radically change the relative
ordering of individuals within a protected group.
In Proposition \ref{prop:counterexample}, 
we show that there is
a causal model where counterfactual
fairness is mutually exclusive with Definition
\ref{def:preserve}.
This suggests that for some causal models, achieving counterfactual
fairness and preserving relative group ordering are incompatible.

Note that Proposition \ref{prop:counterexample},
combined with the fact that Listing 1 and 2
satisfy demographic parity and preserve group order,
does not refute Theorem \ref{thm:dpimplescf}
since the method of estimating latent variables in the proof
of the theorem uses a trivial causal model.

\begin{proposition}\label{prop:counterexample}
    There is a causal model and predictor
    where a counterfactually
    fair algorithm does not maintain the relative
    ordering of predictions 
    \textnormal{within} each group.\footnote{An earlier version of this work gave an incorrect proof
    of the proposition.}
\end{proposition}

\begin{table}[h!]
\begin{tabular}{l|llll|llll|}
\cline{2-9}
\textbf{}                   & \multicolumn{4}{c|}{Short}                                                     & \multicolumn{4}{l|}{Tall}                                                      \\ \cline{2-9} 
                            & \multicolumn{2}{l|}{World A}                    & \multicolumn{2}{l|}{World B} & \multicolumn{2}{l|}{World A}                    & \multicolumn{2}{l|}{World B} \\ \cline{2-9} 
                            & \multicolumn{1}{l|}{$u$} & \multicolumn{1}{l|}{$\hat{y}$} & \multicolumn{1}{l|}{$u$}  & $\hat{y}$  & \multicolumn{1}{l|}{$u$} & \multicolumn{1}{l|}{$\hat{y}$} & \multicolumn{1}{l|}{$u$}  & $\hat{y}$  \\ \hline
\multicolumn{1}{|l|}{Teal}  & \multicolumn{1}{l|}{0} & \multicolumn{1}{l|}{0} & \multicolumn{1}{l|}{1}  & 1  & \multicolumn{1}{l|}{0} & \multicolumn{1}{l|}{1} & \multicolumn{1}{l|}{1}  & 0  \\ \hline
\multicolumn{1}{|l|}{Lucas} & \multicolumn{1}{l|}{1} & \multicolumn{1}{l|}{1} & \multicolumn{1}{l|}{0}  & 0  & \multicolumn{1}{l|}{1} & \multicolumn{1}{l|}{0} & \multicolumn{1}{l|}{0}  & 1  \\ \hline
\end{tabular}
\caption{Latent variable charm $U$ and predicted outcome $\hat{Y}$ for each person, world, and protected group in the example for Proposition
\ref{prop:counterexample}.}
\label{tab:example}
\end{table}

\begin{proof}
Consider an example with two people, 
say Teal and Lucas.
The protected attribute $A$ is binary height,
the latent variable $U$ is binary charm, 
and the true outcome $Y$ is binary success.
The latent variable for Teal is 
$u_\textrm{Teal} \sim \textrm{Binomial}(1/2)$
while the latent variable for Lucas is
$u_\textrm{Lucas} = 1 - u_\textrm{Teal}$.
In other words,
the latent variable for each person is probabilistic 
but correlated: with probability 1/2, we live in
World A and otherwise, we live in World B.
In World A, Teal is not charming and Lucas is 
while, in World B, Teal is charming and Lucas is not.
Now the predictor function satisfies
$\hat{Y}_{A \leftarrow \textrm{short}}(u) = u$,
$\hat{Y}_{A \leftarrow \textrm{tall}}(u) 
= \texttt{NOT}(u)$
for latent variable charm $u \in \{0,  1\}$.
The output of the predictor in each world is given
in Table \ref{tab:example}.

Notice that the relative order of the outcome
produced by the predictor and causal model changes
with the counterfactual intervention.
In World A, Teal is behind Lucas as a short
person and ahead as a tall person.
In World B, Teal is ahead of Lucas as a short
person and behind as a tall person.
We have constructed a causal model and predictor
but it remains to show that they satisfy
counterfactual fairness and that the latent
variable is independent of protected attributes.

First, we show that the causal model and predictor
are counterfactually fair.
Recall that World A and World B are equally likely.
As a result,
by inspecting Table \ref{tab:example} and the
construction,
we conclude that both Teal and Lucas are
as likely to be successful
as a short person and as a tall person.
Therefore the construction satisfies 
counterfactual fairness.

Second, we show that the latent variable is independent
of the protected attribute.
That is, we want to show $\Pr(u,a) = \Pr(u) \Pr(a)$
for $u, a \in \{0, 1\}$ in our example.
Again, by inspecting Table \ref{tab:example} and the
construction, we see that this holds.
For example,
the probability a person is charming and tall is 2/8
while the probability a person is charming is 4/8 and
the probability a person is tall is 4/8.
The same holds for the other three possibilities.
\end{proof}

\paragraph*{Additional Experiments}
So far we have reported on the extensive 
empirical experiments we ran on the law school data set.
We corroborate our findings by running the same set of experiments on two more examples.
The first example is data from the Home
Credit Default Risk data set\footnote{\href{https://www.kaggle.com/c/home-credit-default-risk}{https://www.kaggle.com/c/home-credit-default-risk}}.
We simulated a ``domain expert'' process of causal analysis and constructed a credit risk
causal model by hand (using observations from the data).
The second example is synthetic data
generated \textit{directly} from a causal model we built, 
modeling a hypothetical healthcare context
loosely motivated by \cite{obermeyer2019dissecting}, and represents the ``ideal'' circumstance where the causal model posited over the data is certainly correct.
We show the causal models in Figure \ref{fig:othercausalmodels}.
The remaining figures and results 
are included for reproducibility in the the same 
\href{https://github.com/lurosenb/simplifying_counterfactual_fairness}
{\textbf{repository}}
as the rest of our results.

\begin{figure}[ht!]
    \centering
    \includegraphics[scale=.2]{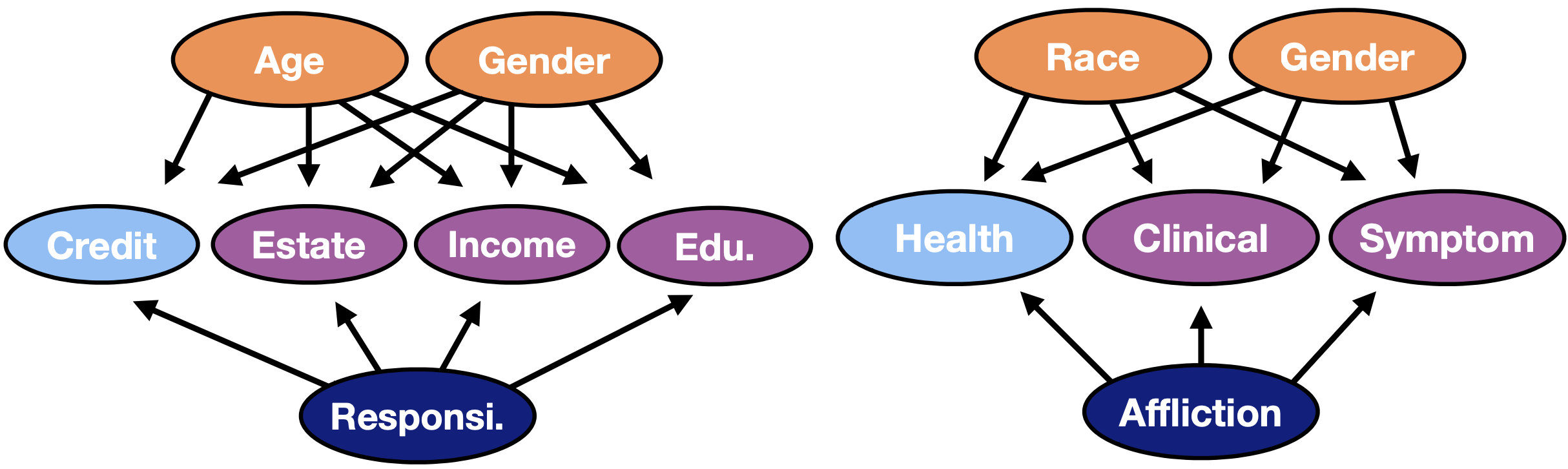}
    \caption{The causal diagrams for the default risk example (on the left) and healthcare example (on the right).}
    \label{fig:othercausalmodels}
\end{figure}

\section{Conclusion}\label{sec:discuss}
Counterfactual fairness has been
celebrated as novel and promising while
demographic parity has been treated as simple and flawed.
Our work shows how the two definitions of fairness are
basically equivalent when considered across an entire population.
We think it is a promising avenue for future work
to identify similarities between other definitions of fairness.
We also introduce a method for comparing algorithms
at the level of individual predictions.
Our findings suggest a natural definition
of preserving group orderings which can be mutually
exclusive 
with counterfactual fairness.
We leave open the problem of designing
novel algorithms that are competitive with simple
benchmarks in terms
of fairness, performance, and preserving group orderings.

\section*{Acknowledgements}
We thank Apoorv Vikram Singh
for his insightful
comments on an early draft of this work.
\bibliography{references}

\end{document}